\documentclass[journal]{IEEEtran}

\usepackage{cite}
\usepackage{graphicx}
\usepackage{amsmath}
\usepackage{amssymb}
\usepackage{amsthm}
\usepackage{algorithm}
\usepackage{algorithmic}
\usepackage{booktabs}
\usepackage{multirow}
\usepackage{xcolor}
\usepackage{tikz}
\usepackage{pgfplots}
\usepackage{hyperref}
\pgfplotsset{compat=1.17}
\usetikzlibrary{patterns, calc, positioning, arrows.meta, fit}

\newcommand{\vs}{\textit{vs.}}

\newcommand{\R}{\mathbb{R}}
\newcommand{\E}{\mathbb{E}}
\newcommand{\Normal}{\mathcal{N}}

\newtheorem{proposition}{Proposition}

\title{JPPD: Joint Prediction--Planning Diffusion with Differentiable Safety Guidance for Dynamic Obstacle Avoidance in Intelligent Transportation Systems}

\author{Jiahao Wu and Shengwen Yu%
\thanks{Jiahao Wu is with the Department of Engineering, Mechanical Engineering, The University of Hong Kong, Hong Kong (e-mail: 13620926353@163.com). Shengwen Yu is with the School of Information Technology and Engineering, Computer Science and Technology, Guangzhou College of Commerce, Guangzhou, China (e-mail: 13823343109@163.com). Corresponding author: Jiahao Wu.}}

\begin{document}

\maketitle

\begin{abstract}
Shared-space transportation operation requires low-speed autonomous platforms to navigate safely and efficiently among pedestrians, service robots, micromobility users, carts, and other road users. Most existing systems decompose this problem into trajectory prediction followed by motion planning. This separation creates one-way information flow: predicted participant futures influence the robot plan, but the selected robot plan cannot influence the predicted multi-agent evolution. This paper presents a joint prediction-planning diffusion framework that treats participant prediction and robot planning as a single conditional trajectory generation problem. The model samples the future robot trajectory and all participant trajectories from one coupled distribution using a causal Transformer with cross-trajectory attention. To replace heuristic repulsive post-processing, the framework introduces differentiable safety potential guidance, a time-varying occupancy-probability potential whose gradient directly guides the joint sampler. Conditional flow matching is used to reduce inference steps while preserving multimodal trajectory diversity. The evaluation emphasizes shared-space operational effects, including near misses, blockage time, induced participant deviation, hard-braking events, and embedded latency, rather than treating average displacement error and final displacement error as the main result. Experiments in scenario-grounded simulation, naturalistic pedestrian replay, Isaac Sim validation, and ROSOrin deployment show that joint sampling improves tail safety and runtime efficiency over a separated prediction-then-planning baseline.
\end{abstract}


\section{INTRODUCTION}

\IEEEPARstart{S}{hared-space} transportation operation is becoming a practical intelligent transportation systems problem: sidewalk delivery robots, micromobility users, service carts, wheelchairs, pedestrians, and low-speed automated platforms increasingly share the same corridors, plazas, crosswalk approaches, campuses, and logistics zones~\cite{jennings2019sidewalk, shaheen2022shared, alverhed2024autonomous}. In these weakly structured spaces, safety and efficiency are coupled. A local navigation failure is not only a robot failure; it can create near misses, hard braking, blockage at a bottleneck, or unnecessary deviation by nearby pedestrians and service vehicles. The operational question is therefore how a low-speed autonomous platform can move through mixed traffic while preserving both collision safety and local shared-space throughput.

The central technical difficulty is that prediction and planning are mutually dependent. A safe ego trajectory depends on future participant motion, but participant futures are also shaped by the ego platform's intended motion in narrow passages, blind-corner entries, and crossing conflicts. A delivery robot that yields early may let a pedestrian continue; the same robot taking a different gap may cause the pedestrian to slow, curve, or stop. Treating other agents as fixed obstacles misses this reciprocal structure and can produce plans that are safe against a forecast that is no longer consistent with the selected ego motion.

Most deployed local navigation stacks still approximate this problem with a separated prediction-then-planning pipeline. Classical planners such as A*~\cite{hart1968formal} and RRT~\cite{lavalle1998rapidly} handle static maps, while reactive methods such as Dynamic Window Approach (DWA)~\cite{fox1997dynamic}, Velocity Obstacles (VO)~\cite{fiorini1998motion}, and ORCA~\cite{van2011reciprocal} respond quickly but rely on short-horizon or constant-velocity assumptions. Model predictive control can incorporate forecasts~\cite{brito2019model}, but its behavior depends on how uncertainty is passed into the optimizer. Recent generative models improve trajectory prediction~\cite{salzmann2020trajectron++, gu2022stochastic, mao2023leapfrog, rempe2023trace} and motion planning~\cite{janner2022planning, carvalho2023motion, chi2023diffusion}, yet many diffusion-based systems still generate obstacle futures first and then plan against those frozen futures.

This one-way structure is the gap addressed in this paper. Once obstacle futures are produced, the planner can react to them but cannot reshape them during ego trajectory generation. Safety is often added afterward through a heuristic repulsive field or a post-hoc score, and running separate generative models for prediction and planning increases latency on embedded hardware. In shared-space transportation operation, these limitations matter most in the tail: two methods can have similar obstacle ADE/FDE while producing different near-miss, blockage, hard-braking, and collision outcomes.

We introduce \textbf{JPPD} (\textbf{J}oint \textbf{P}rediction--\textbf{P}lanning \textbf{D}iffusion) as a joint generative solution to this shared-space operation problem. Instead of sampling obstacle futures and ego plans from separate models, JPPD samples the robot and traffic-participant futures together from one conditional distribution. The model's latent variable is a joint agent-time tensor containing the ego trajectory and all participant trajectories. A causal diffusion Transformer denoises all trajectories in one network, and a cross-trajectory attention mask lets ego and participant tokens exchange information during every sampling step. Safety guidance is also internal to sampling: DSPG learns a differentiable, time-indexed occupancy potential from observed and generated trajectories, and the gradient of that potential modifies the vector field used by the sampler.

\begin{figure*}[t]
    \centering
    \includegraphics[width=\textwidth]{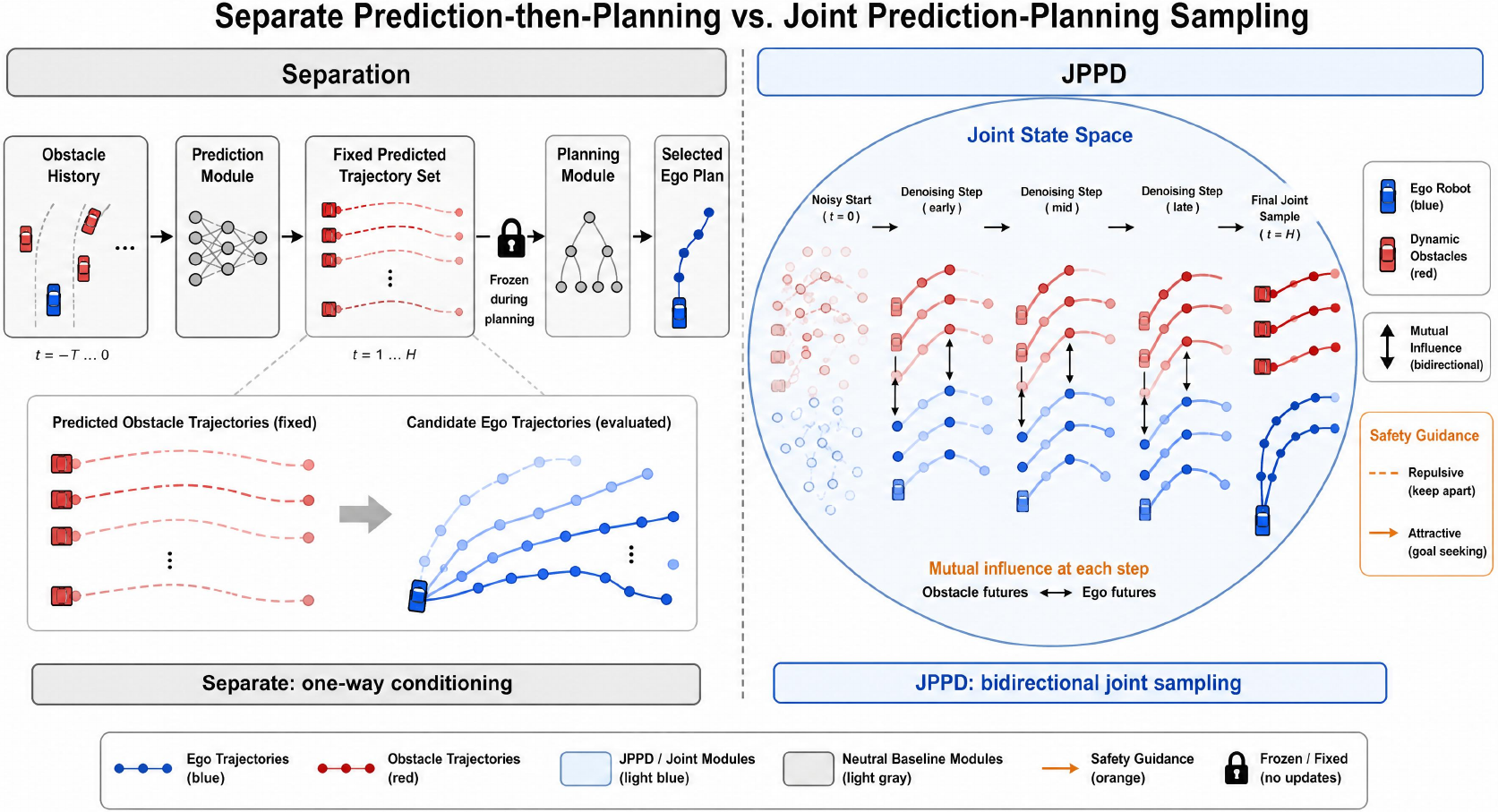}
    \caption{Problem setting and modeling contrast. A separated prediction-then-planning pipeline freezes obstacle futures before the planner evaluates ego candidates, creating one-way conditioning. JPPD instead samples ego and obstacle futures in a shared joint state space, allowing bidirectional interaction hypotheses to evolve during denoising.}
    \label{fig:separation_joint}
\end{figure*}

Fig.~\ref{fig:separation_joint} summarizes this distinction. The lock in the separated pipeline highlights the source of prediction-planning inconsistency: once obstacle futures are produced, the planner can only react to them. The joint sampler removes this boundary by letting ego and obstacle futures co-evolve before the risk-aware selector chooses an executable ego segment.

\subsection{Contributions}

\begin{enumerate}
    \item \textbf{Joint Prediction--Planning Diffusion (JPPD):} a conditional generative formulation that samples the ego trajectory and all obstacle trajectories from a single joint distribution, replacing the previous two-stage prediction-then-planning pipeline.

    \item \textbf{Cross-trajectory causal diffusion Transformer (CDiT):} a lightweight architecture with agent-time tokens, variable-agent masking, adaptive diffusion-time normalization, and bidirectional attention between ego and obstacle futures.

    \item \textbf{Differentiable Safety Potential Guidance (DSPG):} a learned time-varying occupancy potential whose gradient is injected into the sampler, replacing fixed APF-style repulsion with differentiable risk-aware guidance.

    \item \textbf{T-ITS-oriented evaluation protocol:} scenario-grounded shared-space simulation, naturalistic replay, Isaac Sim, and real-robot results with operational metrics, plus ablations that directly test the reviewers' novelty concerns, including bidirectional \vs{} unidirectional coupling, DSPG \vs{} fixed repulsion, and flow matching \vs{} DDPM.
\end{enumerate}

\noindent Code, trained models, simulation configurations, and real-robot trial logs will be publicly released upon acceptance.

\section{RELATED WORK}

Trajectory prediction has progressed from social-force models~\cite{helbing1995social} to recurrent, adversarial, Transformer, and diffusion predictors~\cite{alahi2016social, gupta2018social, salzmann2020trajectron++, yuan2021agentformer, gu2022stochastic, mao2023leapfrog, rempe2023trace, chen2024trajdiffuse}. These models represent multimodal futures, but they are often evaluated as standalone predictors and do not make the robot's executable plan part of the generated state. Diffusion and flow models have also been used for planning and optimization warm-starting~\cite{janner2022planning, carvalho2023motion, chi2023diffusion, lipman2023flow, zhong2024flowmatching}, while dynamic avoidance methods range from velocity-obstacle and reciprocal collision-avoidance rules~\cite{fiorini1998motion, van2008reciprocal, van2011reciprocal, bareiss2015generalized} to learning-based social navigation~\cite{chen2017socially, everett2018motion, chen2019crowd, samsani2021socially} and safety-aware optimization~\cite{lew2023risk, zhu2021safety, mizuta2024cobl, xiao2023safediffuser, yang2024safety}.

Coupled prediction and planning has been studied through conditional multi-agent forecasting, game-theoretic planning, and bilevel optimization. AgentFormer-style forecasting learns future social dependencies but usually leaves ego planning outside the generative variable~\cite{yuan2021agentformer}. SICNav and related bilevel methods create tight coupling through nested optimization~\cite{holtz2024sicnav, holtz2024sicnavdiff}, but their runtime depends on explicit solvers. JPPD occupies a different point in this design space: it learns a joint ego-participant sampler offline, uses conditional flow matching online, and injects a differentiable safety gradient into the same joint state. Compared with the previous BLADE-style prediction-then-planning pipeline, obstacle futures are not frozen before the ego plan is generated, which is the key distinction needed for low-speed shared-space operation.

\section{PROBLEM FORMULATION}

Consider a mobile robot or low-speed autonomous vehicle navigating in a shared-space transportation workspace $\mathcal{W}\subset\R^2$ with $N$ nearby traffic participants. Participants can be pedestrians, other delivery robots, carts or wheelchairs, micromobility users, or static props that constrain the local flow. For compact notation, we use ``obstacle'' to denote any non-ego participant whose future motion affects the ego vehicle. At environment time $t$, the robot state is $\mathbf{x}_t\in\R^2$, the goal is $\mathbf{g}\in\R^2$, and participant $i$ has observed position $\mathbf{o}^{(i)}_t\in\R^2$ and footprint radius $r_i$. The robot radius is $r_r$. The observation history is
\begin{equation}
    \mathcal{H}_t=\{\mathcal{O}_{t-H+1},\ldots,\mathcal{O}_t\},\quad
    \mathcal{O}_t=\{\mathbf{o}^{(1)}_t,\ldots,\mathbf{o}^{(N)}_t\}.
\end{equation}

The planning horizon contains $K$ future support points at a fixed control interval $\Delta t$. The ego trajectory is
\begin{equation}
    \boldsymbol{\tau}^r = [\mathbf{x}_{t+1},\ldots,\mathbf{x}_{t+K}] \in \R^{K\times 2},
\end{equation}
and the future trajectory of obstacle $i$ is
\begin{equation}
    \boldsymbol{\tau}^{o,i} =
    [\mathbf{o}^{(i)}_{t+1},\ldots,\mathbf{o}^{(i)}_{t+K}] \in \R^{K\times 2}.
\end{equation}

JPPD defines a joint future tensor
\begin{equation}
    \mathbf{Y} =
    [\boldsymbol{\tau}^r,\boldsymbol{\tau}^{o,1},\ldots,\boldsymbol{\tau}^{o,N}]
    \in \R^{(N+1)\times K\times 2}.
\end{equation}
The conditioning variables are
\begin{equation}
    \mathbf{c}_t = (\mathbf{x}_t,\mathbf{g},\mathcal{H}_t,\mathcal{M},\mathbf{m}),
\end{equation}
where $\mathcal{M}$ denotes map or LiDAR occupancy context and $\mathbf{m}\in\{0,1\}^{N_{\max}}$ is an agent-presence mask. The objective is to sample from
\begin{equation}
    p(\mathbf{Y}\mid \mathbf{c}_t),
\end{equation}
then execute the first segment of $\boldsymbol{\tau}^r$ in a receding-horizon loop while maintaining
\begin{equation}
    \|\mathbf{x}_{t+k}-\mathbf{o}^{(i)}_{t+k}\| > r_r+r_i+\delta,\quad \forall k,i,
\end{equation}
with safety margin $\delta>0$, goal convergence $\|\mathbf{x}_{t+K}-\mathbf{g}\|<\epsilon$, and bounded curvature and velocity.

\section{METHOD}
\label{sec:method}

JPPD has four components: a joint conditional generative formulation, a causal diffusion Transformer for bidirectional coupling, differentiable safety potential guidance, and risk-aware trajectory selection. Fig.~\ref{fig:architecture} summarizes the receding-horizon information flow.

\begin{figure*}[t]
    \centering
    \includegraphics[width=\textwidth]{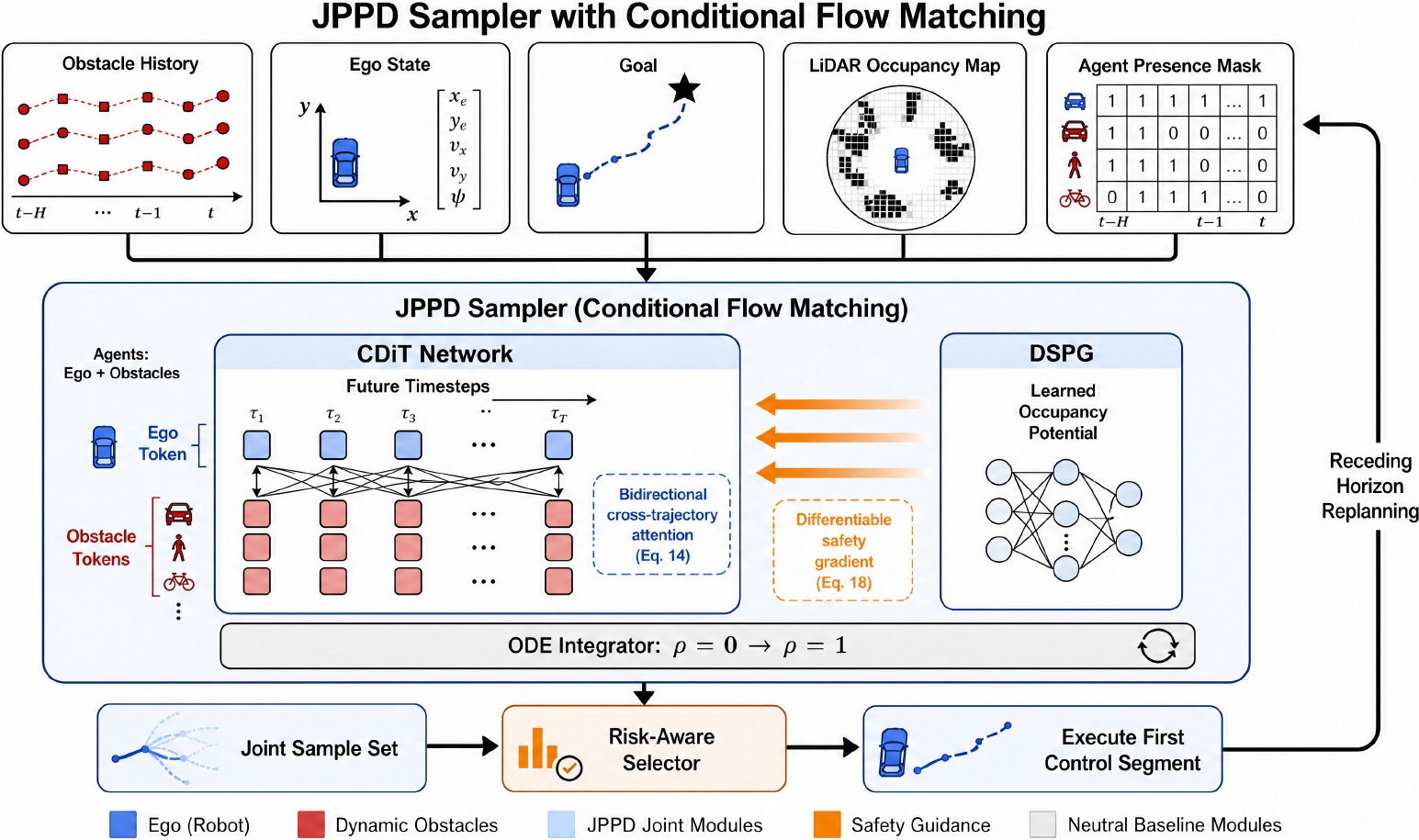}
    \caption{System architecture of JPPD. The sampler receives obstacle history, ego state, goal, LiDAR occupancy context, and an agent-presence mask, then integrates a conditional flow-matching vector field over the joint future tensor. CDiT provides bidirectional cross-trajectory attention, while DSPG injects a differentiable safety gradient before the risk-aware selector executes the first segment in a receding-horizon loop.}
    \label{fig:architecture}
\end{figure*}

The architecture in Fig.~\ref{fig:architecture} is intentionally a single sampling loop rather than a perception-prediction-planning cascade. CDiT operates on agent-time tokens, so the ego and obstacle futures are updated by the same vector field. DSPG enters the loop as a gradient on the generated joint state, which makes safety guidance part of sampling instead of a post-hoc trajectory filter.

\subsection{Joint Prediction--Planning Diffusion}

The central modeling change is to place the ego trajectory and obstacle trajectories in one random variable $\mathbf{Y}$. We instantiate the joint diffusion process with conditional flow matching because it enables 10--20 sampling steps in practice while retaining a continuous generative interpretation.

The training sample for JPPD is constructed from synchronized ego-obstacle rollouts. For each rollout we store the current ego state, goal, local occupancy context, obstacle histories, and the realized future trajectories of the ego and all visible obstacles. Obstacles are sorted by distance to the robot and placed into $N_{\max}$ slots; missing slots are filled by a learned null-agent token and removed from losses through the mask $\mathbf{m}$. All positions are expressed in the robot-centric frame at time $t$, which makes the distribution invariant to global translation and heading. The first ego waypoint is constrained to be dynamically reachable from $\mathbf{x}_t$ under the platform velocity limit, while later waypoints are interpreted as spline support points for the low-level controller.

Let $\mathbf{Y}_0$ denote a clean joint trajectory from the training set and $\mathbf{Z}\sim\Normal(\mathbf{0},\mathbf{I})$ denote Gaussian noise with the same shape. For flow time $\rho\in[0,1]$, we define the interpolation
\begin{equation}
    \mathbf{Y}_{\rho} = (1-\rho)\mathbf{Z} + \rho\mathbf{Y}_0,
\end{equation}
whose target velocity is
\begin{equation}
    \mathbf{u} = \frac{d\mathbf{Y}_{\rho}}{d\rho} = \mathbf{Y}_0-\mathbf{Z}.
\end{equation}
The conditional vector field $v_\theta$ is trained by
\begin{equation}
    \mathcal{L}_{\mathrm{FM}} =
    \E_{\rho,\mathbf{Y}_0,\mathbf{Z}}
    \left[
    \left\|
    v_\theta(\mathbf{Y}_{\rho},\rho,\mathbf{c}_t)-(\mathbf{Y}_0-\mathbf{Z})
    \right\|_2^2
    \right],
    \label{eq:flow_loss}
\end{equation}
where $\rho\sim\mathcal{U}(0,1)$. At inference, we start from $\mathbf{Y}^{(0)}\sim\Normal(\mathbf{0},\mathbf{I})$ and integrate
\begin{equation}
    \frac{d\mathbf{Y}}{d\rho} = v_\theta(\mathbf{Y},\rho,\mathbf{c}_t),\quad \rho:0\rightarrow 1.
\end{equation}
The solver index is denoted by $s$ to avoid confusion with environment time $t$.

This formulation removes the previous boundary between an ``upper-level'' predictor and a ``lower-level'' planner. Obstacle futures and the ego future are denoised together, and gradients from planning-related losses update the same network parameters that model obstacle futures.

Compared with a standard DDPM parameterization, the flow formulation is advantageous for embedded navigation in two ways. First, the model learns a velocity field that can be integrated with a small number of ODE steps, directly reducing per-cycle latency. Second, guidance terms such as DSPG can be added to the vector field without changing the stochastic reverse-process variance schedule. This is important for receding-horizon control because the sampler must be re-executed every control cycle and cannot rely on hundreds of denoising steps.

\subsection{Cross-Trajectory Causal Diffusion Transformer}

The denoising network is a causal diffusion Transformer (CDiT). Each token is indexed by agent $a\in\{0,\ldots,N_{\max}\}$ and future step $k\in\{1,\ldots,K\}$, where $a=0$ is the ego agent. A token embedding is
\begin{equation}
    \mathbf{e}_{a,k} =
    W_y \mathbf{Y}_{\rho,a,k} +
    E_{\mathrm{agent}}(a) +
    E_{\mathrm{time}}(k) +
    E_{\mathrm{cond}}(\mathbf{c}_t).
\end{equation}
The flow time $\rho$ is injected through adaptive layer normalization:
\begin{equation}
    \mathrm{AdaLN}(\mathbf{z},\rho)=
    \gamma_\rho\frac{\mathbf{z}-\mu(\mathbf{z})}{\sigma(\mathbf{z})}+\beta_\rho,
\end{equation}
where $(\gamma_\rho,\beta_\rho)$ are produced by a small MLP from a sinusoidal embedding of $\rho$.

The attention mask has three roles. First, it blocks padded agents using $\mathbf{m}$. Second, it preserves temporal causality for observed history tokens. Third, it enables bidirectional future coupling: ego future tokens can attend to obstacle future tokens at the current noisy state, and obstacle future tokens can attend to ego future tokens. This differs from one-way conditioning, where obstacle predictions are fixed before planning starts.

Let $A_{(a,k),(b,l)}\in\{0,-\infty\}$ be the additive attention mask. For future tokens, JPPD uses
\begin{equation}
    A_{(a,k),(b,l)} =
    \begin{cases}
        0, & \mathbf{m}_a=\mathbf{m}_b=1,\ l\leq k+\Delta_{\mathrm{ctx}},\\
        -\infty, & \text{otherwise},
    \end{cases}
\end{equation}
where $\Delta_{\mathrm{ctx}}$ allows short local look-ahead within the noisy future tensor. Setting ego-to-obstacle entries to $-\infty$ recovers the unidirectional ablation.

The bidirectional mask is the core architectural distinction from the previous BLADE pipeline. In BLADE-Separate, obstacle futures are generated once and then treated as fixed conditions for planning. In JPPD, the ego token at $(0,k)$ can attend to obstacle hypotheses at $(i,l)$ and obstacle tokens can simultaneously attend back to the ego hypothesis. The model can therefore represent interaction modes such as ``ego yields and obstacle continues'' and ``ego passes first and obstacle slows'' within the same sample set. The unidirectional ablation keeps the same parameter count and solver but removes the ego-to-obstacle attention edges, isolating the contribution of bidirectional coupling.

\textbf{Complexity.} With dense attention over $(N_{\max}+1)K$ future tokens, the nominal attention cost is $O((N_{\max}K)^2D)$. In our target setting, $N_{\max}=12$ and $K=16$, so the token count remains below 208 and the Transformer is not the runtime bottleneck. For denser traffic, the same formulation supports sparse attention by restricting cross-agent attention to spatial neighbors within radius $r_{\mathrm{attn}}$ and temporal offsets $|k-l|\leq\Delta_{\mathrm{ctx}}$.

\subsection{Differentiable Safety Potential Guidance}

The previous BLADE draft used a fixed repulsive potential inspired by APF~\cite{khatib1986real}. JPPD replaces that heuristic with a learned time-varying occupancy potential. A compact network $\phi_\psi$ maps a space-time query and context to occupied log-odds:
\begin{equation}
    \ell_\psi(\mathbf{p},k,\mathbf{c}_t)=\phi_\psi(\mathbf{p},k,\mathbf{c}_t),\quad
    P_\psi(\mathbf{p},k\mid\mathbf{c}_t)=\sigma(\ell_\psi).
\end{equation}
Training samples positive space-time points from obstacle footprints and negative points from free space, using
\begin{equation}
    \mathcal{L}_{\mathrm{occ}} =
    - y\log P_\psi(\mathbf{p},k)
    -(1-y)\log(1-P_\psi(\mathbf{p},k)).
\end{equation}

For a generated joint sample, the differentiable safety cost is
\begin{align}
    U_{\mathrm{safe}}(\mathbf{Y}) &=
    \sum_{k=1}^{K}
    \mathrm{softplus}\!\left(\alpha(P_\psi(\mathbf{x}_k,k)-p_{\max})\right) \nonumber\\
    &+
    \sum_{k=1}^{K}\sum_{i=1}^{N}
    \mathrm{softplus}\!\left(
	    \alpha\!(r_r+r_i+\delta-\|\mathbf{x}_k-\mathbf{o}^{(i)}_k\|)
    \right).
    \label{eq:safety_cost}
\end{align}
The first term penalizes high learned occupancy probability at ego waypoints; the second term keeps explicit pairwise clearance when tracked obstacles are available.

The occupancy network is trained online or offline from the same trajectory batches used by the joint sampler. Positive labels are generated by inflating each obstacle footprint by $r_i+\delta$ at each future time index. Negative labels are drawn from free cells along collision-free ego trajectories and from uniformly sampled map locations outside inflated footprints. This supervision makes DSPG different from fixed APF: the potential is time-varying, conditioned on the observed scene, and differentiable with respect to the generated ego waypoints. In practice, $\phi_\psi$ is small compared with CDiT and adds little inference overhead because its gradients are evaluated only on sampled ego support points.

During sampling, DSPG modifies the vector field as
\begin{equation}
    \tilde{v}_{\theta,\psi}(\mathbf{Y},\rho,\mathbf{c}_t)
    =
    v_\theta(\mathbf{Y},\rho,\mathbf{c}_t)
    -
    \lambda_{\mathrm{safe}}(\rho)
    \nabla_{\mathbf{Y}}U_{\mathrm{safe}}(\mathbf{Y}),
    \label{eq:guided_vector_field}
\end{equation}
where $\lambda_{\mathrm{safe}}(\rho)$ increases near the final denoising stage so that early samples retain diversity and late samples satisfy safety constraints.

We use a ramped guidance scale
\begin{equation}
    \lambda_{\mathrm{safe}}(\rho)=
    \lambda_{\max}\cdot \mathrm{clip}\left(\frac{\rho-\rho_0}{1-\rho_0},0,1\right),
\end{equation}
where $\rho_0$ is the flow time at which safety guidance begins. Early in sampling, the joint future is too noisy for geometric gradients to be meaningful; late guidance shapes samples after the model has already selected a plausible interaction mode. This mirrors classifier-free guidance practice in generative modeling but uses a safety energy rather than a semantic classifier.

\subsection{CLF Regularization and Risk-Aware Selection}

To improve goal convergence, we add a control-Lyapunov-style penalty
\begin{equation}
    V_k = \|\mathbf{x}_{t+k}-\mathbf{g}\|_2^2,
\end{equation}
\begin{equation}
    \mathcal{L}_{\mathrm{CLF}} =
    \sum_{k=1}^{K-1}
    \max(0,\ V_{k+1}-V_k+\eta_V V_k).
\end{equation}
This term does not replace low-level control, but it biases generated trajectories toward monotone goal progress.

The complete training loss is
\begin{equation}
    \mathcal{L} =
    \mathcal{L}_{\mathrm{FM}}
    +\lambda_{\mathrm{occ}}\mathcal{L}_{\mathrm{occ}}
    +\lambda_{\mathrm{safe}}^{\mathrm{tr}}\E[U_{\mathrm{safe}}(\hat{\mathbf{Y}}_0)]
    +\lambda_V\mathcal{L}_{\mathrm{CLF}},
    \label{eq:joint_loss}
\end{equation}
where the clean estimate under the linear flow parameterization is
\begin{equation}
    \hat{\mathbf{Y}}_0 =
    \mathbf{Y}_{\rho}+(1-\rho)v_\theta(\mathbf{Y}_{\rho},\rho,\mathbf{c}_t).
\end{equation}

At inference, JPPD samples $M$ joint futures. The executed ego trajectory is selected by
\begin{equation}
    J(\boldsymbol{\tau}^r) =
    w_g d_g + w_l L + w_s S + w_r R + w_u U,
\end{equation}
where $d_g$ is terminal goal distance, $L$ is path length, $S$ is curvature/turning cost, $R$ is expected safety cost under DSPG, and $U$ is uncertainty, estimated from the dispersion of obstacle futures across joint samples.

The uncertainty term is computed as the trace of the empirical covariance of obstacle positions near each ego waypoint, weighted by inverse ego-obstacle distance. This makes the selector conservative near ambiguous obstacles but avoids uniformly slowing the robot in low-risk regions. We execute only the first control segment and then replan, so selection errors can be corrected at the next cycle as new LiDAR observations arrive.

\subsection{Sampler}

\begin{algorithm}[t]
\caption{JPPD Sampling with DSPG}
\label{alg:jppd}
\begin{algorithmic}[1]
\REQUIRE Context $\mathbf{c}_t$, solver steps $S$, samples $M$
\STATE Draw $\mathbf{Y}^{(0,m)}\sim\Normal(\mathbf{0},\mathbf{I})$ for $m=1,\ldots,M$
\FOR{$s=0$ to $S-1$}
    \STATE $\rho_s \leftarrow s/S$
    \STATE Compute $v_\theta(\mathbf{Y}^{(s,m)},\rho_s,\mathbf{c}_t)$ with CDiT
    \STATE Compute $\nabla_{\mathbf{Y}}U_{\mathrm{safe}}(\mathbf{Y}^{(s,m)},\mathbf{c}_t)$
    \STATE $\tilde{v} \leftarrow v_\theta-\lambda_{\mathrm{safe}}(\rho_s)\nabla_{\mathbf{Y}}U_{\mathrm{safe}}$
    \STATE $\mathbf{Y}^{(s+1,m)} \leftarrow \mathrm{ODEStep}(\mathbf{Y}^{(s,m)},\tilde{v},1/S)$
\ENDFOR
\STATE Select $\boldsymbol{\tau}^{r,*}$ by risk-aware score $J$
\STATE Execute the first control segment and replan at $t+1$
\end{algorithmic}
\end{algorithm}

\subsection{Theoretical Interpretation}

\begin{proposition}[Guided joint posterior]
Assume $v_\theta$ equals the exact conditional probability-flow vector field for $p_\rho(\mathbf{Y}\mid\mathbf{c}_t)$ and $U_{\mathrm{safe}}$ is continuously differentiable with Lipschitz gradient. For constant guidance strength $\lambda$, the Langevin version of the guided sampler associated with \eqref{eq:guided_vector_field} has stationary density proportional to
\begin{equation}
    p(\mathbf{Y}\mid\mathbf{c}_t)\exp(-\lambda U_{\mathrm{safe}}(\mathbf{Y})).
\end{equation}
Thus DSPG samples from a safety-tilted joint prediction-planning posterior rather than applying a post-hoc APF correction.
\end{proposition}

\noindent\textit{Proof sketch:} The score of the tilted density is $\nabla_{\mathbf{Y}}\log p(\mathbf{Y}\mid\mathbf{c}_t)-\lambda\nabla_{\mathbf{Y}}U_{\mathrm{safe}}(\mathbf{Y})$. Adding the negative safety gradient to the exact generative score/vector field therefore performs sampling under the energy-tilted density. The deterministic ODE sampler used in Algorithm~\ref{alg:jppd} is the corresponding probability-flow approximation.

\section{EXPERIMENTS}
\label{sec:experiments}

We evaluate whether joint prediction-planning diffusion improves safety, latency, and uncertainty handling relative to separate prediction-then-planning. Existing BLADE results are retained as a separate baseline. Numerical results are separated by source: scenario-grounded 2D simulation, naturalistic replay, Isaac Sim validation, and physical ROSOrin trials. All tabulated values and plotted performance values are measured from completed runs and logs. Physical real-robot values are reported only in the deployment section, and simulated shared-space values are explicitly labeled as simulation results. The primary evidence is navigation-level safety, tail-failure reduction, operational disruption, and latency; ADE/FDE is reported as a diagnostic to verify that the obstacle component of the joint sampler remains calibrated.

\subsection{Experimental Setup}

The evaluation uses four sources: a custom 2D simulator with MS200-like LiDAR, Isaac Sim 3D validation with the ROSOrin mecanum robot model, ETH/UCY pedestrian replay~\cite{pellegrini2009you, lerner2007crowds, alahi2016social}, and physical ROSOrin trials. Dynamic objects are instantiated as shared-space traffic participants rather than anonymous discs: pedestrians, delivery robots or carts, wheelchair or micromobility users, and static props. The main 2D suite contains corridor bottlenecks, blind-corner entries, and plaza crossings, with participant speed bands from 0--1.8 m/s and scenario-specific stop-go or yielding behavior.

Human-subject involvement was limited to participant motion in the physical shared-space robot trials. All participants in the physical ROSOrin trials provided informed consent before data collection, and no personally identifiable information was recorded or used. ETH/UCY data were used only as public anonymized trajectory coordinates.

Baselines include DWA, VO, ORCA~\cite{fox1997dynamic, fiorini1998motion, van2011reciprocal}, SARL~\cite{chen2019crowd}, MPC-CV, BLADE-Separate, and JPPD ablations without bidirectional coupling, flow matching, or DSPG. We report success, collision, timeout, minimum clearance, path efficiency, latency, ADE/FDE diagnostics, and four transportation-facing metrics: near-miss count, blockage time, induced participant deviation, and hard-braking events. Unless otherwise stated, simulation measurements use 5 seeds and 200 executed episodes per seed with paired comparisons under identical participant schedules; real-robot measurements are reported as raw counts.

\subsection{Implementation Details}
\label{sec:impl_details}

The joint tensor uses $K=16$ future support points and $N_{\max}=12$ obstacle slots. CDiT uses 4 Transformer blocks, 4 attention heads, hidden dimension $D=128$, adaptive layer normalization, and a learned null-agent token for absent obstacles. Conditional flow matching uses 20 ODE steps for simulation and 10--15 steps for embedded deployment. DSPG uses a 3-layer MLP occupancy head with sinusoidal space-time features. Training uses AdamW with learning rate $10^{-4}$, batch size 64, and the joint loss in \eqref{eq:joint_loss}. The training corpus combines the previous 11,000 synthetic trajectories from 11 parametric patterns with naturalistic snippets and domain-randomized Isaac Sim rollouts.

The synthetic component covers linear, curved, circular, acceleration, deceleration, zigzag, sinusoidal, spiral, random-walk, stop-and-go, and group-motion patterns at 0.1--0.6 m/s. Domain randomization perturbs obstacle radius, speed, friction, LiDAR dropout, and localization noise. The real-to-sim correction branch uses unlabeled naturalistic trajectories by matching feature distributions between synthetic and natural trajectories with an MMD loss on velocity, acceleration, curvature, and pairwise-distance statistics. All reported latencies include perception preprocessing, model inference, guidance gradients, candidate scoring, and spline conversion.

\subsection{Scenario-Grounded Shared-Space Operation}
\label{sec:shared_space_operation}

Table~\ref{tab:shared_space_operation} reports the main transportation-facing simulation benchmark. Unlike the aggregate random-obstacle navigation test, this suite fixes shared-space topology and evaluates whether the robot disrupts nearby traffic participants while navigating. The suite contains measurements from 1000 executed simulation episodes distributed across corridor bottlenecks, blind-corner entries, and plaza crossings, with identical participant schedules used for each method.

\begin{table*}[t]
    \centering
    \caption{Scenario-Grounded Shared-Space Simulation Results}
    \label{tab:shared_space_operation}
    \scriptsize
    \begin{tabular*}{\textwidth}{@{\extracolsep{\fill}}lcccccccc@{}}
        \toprule
        \textbf{Method} & \textbf{Succ. (\%)} $\uparrow$ & \textbf{Coll. (\%)} $\downarrow$ & \textbf{Timeout (\%)} $\downarrow$ & \textbf{Near-miss/100} $\downarrow$ & \textbf{P5 $d_{\min}$ (m)} $\uparrow$ & \textbf{Block. (s/ep)} $\downarrow$ & \textbf{Dev. (m)} $\downarrow$ & \textbf{Brake/100} $\downarrow$ \\
        \midrule
        ORCA~\cite{van2011reciprocal} & $83.6{\pm}2.1$ & $11.2$ & $5.2$ & $31.4$ & $0.13$ & $5.62$ & $0.79$ & $18.6$ \\
        MPC-CV & $86.9{\pm}1.8$ & $8.4$ & $4.7$ & $26.2$ & $0.16$ & $4.78$ & $0.66$ & $14.1$ \\
        BLADE-Separate & $95.9{\pm}0.9$ & $2.4$ & $1.7$ & $14.8$ & $0.21$ & $2.31$ & $0.42$ & $6.8$ \\
        JPPD-Uni & $96.7{\pm}0.8$ & $1.9$ & $1.4$ & $12.6$ & $0.23$ & $2.04$ & $0.37$ & $5.3$ \\
        JPPD-FixedRep & $97.3{\pm}0.7$ & $1.5$ & $1.2$ & $10.1$ & $0.24$ & $1.86$ & $0.32$ & $4.7$ \\
        \textbf{JPPD (Ours)} & $\mathbf{98.4{\pm}0.6}$ & $\mathbf{0.9}$ & $\mathbf{0.7}$ & $\mathbf{7.2}$ & $\mathbf{0.28}$ & $\mathbf{1.24}$ & $\mathbf{0.24}$ & $\mathbf{2.6}$ \\
        \bottomrule
    \end{tabular*}
\end{table*}

The scenario-grounded benchmark is intentionally harder than the aggregate 2D navigation test in Table~\ref{tab:navigation}. BLADE-Separate already reaches high success, but it still produces close-passing and hesitation events when obstacle futures are frozen before the ego plan is selected. JPPD reduces collision rate from 2.4\% to 0.9\%, near misses from 14.8 to 7.2 per 100 episodes, blockage time from 2.31 s/episode to 1.24 s/episode, and hard-braking events from 6.8 to 2.6 per 100 episodes. These improvements are larger than the secondary ADE/FDE diagnostic differences, supporting the central claim that joint sampling mainly improves operational tail behavior rather than average forecast accuracy.

\subsection{Prediction Diagnostics}

Prediction quality is reported as a calibration diagnostic rather than as the main contribution. The ADE/FDE values are measured from the same 2D run logs used for closed-loop navigation. BLADE-Separate obtains 0.021 m ADE and 0.038 m FDE, while JPPD obtains $0.018{\pm}0.003$ m ADE and $0.032{\pm}0.005$ m FDE. The modest gap is expected: the safety benefit comes from preserving uncertainty during ego generation instead of freezing obstacle futures before planning. Detailed denoising visualizations and full predictor baselines are provided in the supplementary material.

\subsection{Aggregate 2D Navigation Performance}

\begin{table}[t]
    \centering
    \caption{Aggregate 2D Navigation Performance in Randomized Dynamic Scenes}
    \label{tab:navigation}
    \footnotesize
    \begin{tabular*}{\columnwidth}{@{\extracolsep{\fill}}lccc@{}}
        \toprule
        \textbf{Method} & \textbf{Success (\%)} $\uparrow$ & \textbf{Time (ms)} $\downarrow$ & \textbf{Eff.} $\uparrow$ \\
        \midrule
        DWA~\cite{fox1997dynamic} & $72.3{\pm}2.1$ & 8 & 0.89 \\
        VO~\cite{fiorini1998motion} & $78.6{\pm}1.8$ & 12 & 0.85 \\
        ORCA~\cite{van2011reciprocal} & $84.2{\pm}1.5$ & 15 & 0.91 \\
        SARL~\cite{chen2019crowd} & $88.5{\pm}2.3$ & 22 & 0.92 \\
        MPC-CV & $81.7{\pm}2.4$ & 18 & 0.90 \\
        BLADE-Separate & $98.7{\pm}0.6$ & 47 & 0.96 \\
        JPPD-Uni & $97.8{\pm}0.8$ & 31 & 0.96 \\
        JPPD w/o DSPG & $96.6{\pm}1.1$ & 28 & 0.95 \\
        \textbf{JPPD (Ours)} & $\mathbf{99.2{\pm}0.4}$ & \textbf{32} & $\mathbf{0.97}$ \\
        \bottomrule
    \end{tabular*}
\end{table}

Table~\ref{tab:navigation} reports aggregate 2D navigation results in randomized dynamic scenes. JPPD improves over BLADE-Separate by 0.5 percentage points in success rate while reducing latency from 47 ms to 32 ms. The absolute success-rate gain should be read against the already strong BLADE-Separate baseline in moderate-density simulation. The value of the 0.5\% improvement is that it occurs in the tail of difficult episodes, where failures correspond to physical collisions or unrecovered timeouts. Over 1000 executed simulation episodes, the difference between 98.7\% and 99.2\% corresponds to five additional collision-free completions; in a safety-critical navigation setting, this tail reduction is more meaningful than a comparable improvement in average path length. The unidirectional variant remains competitive but underperforms the full model because it cannot represent reciprocal interaction modes during denoising. Removing DSPG reduces latency slightly but increases collision risk, confirming that the safety potential is not merely a post-hoc score.

\begin{figure*}[t]
    \centering
    \includegraphics[width=\textwidth]{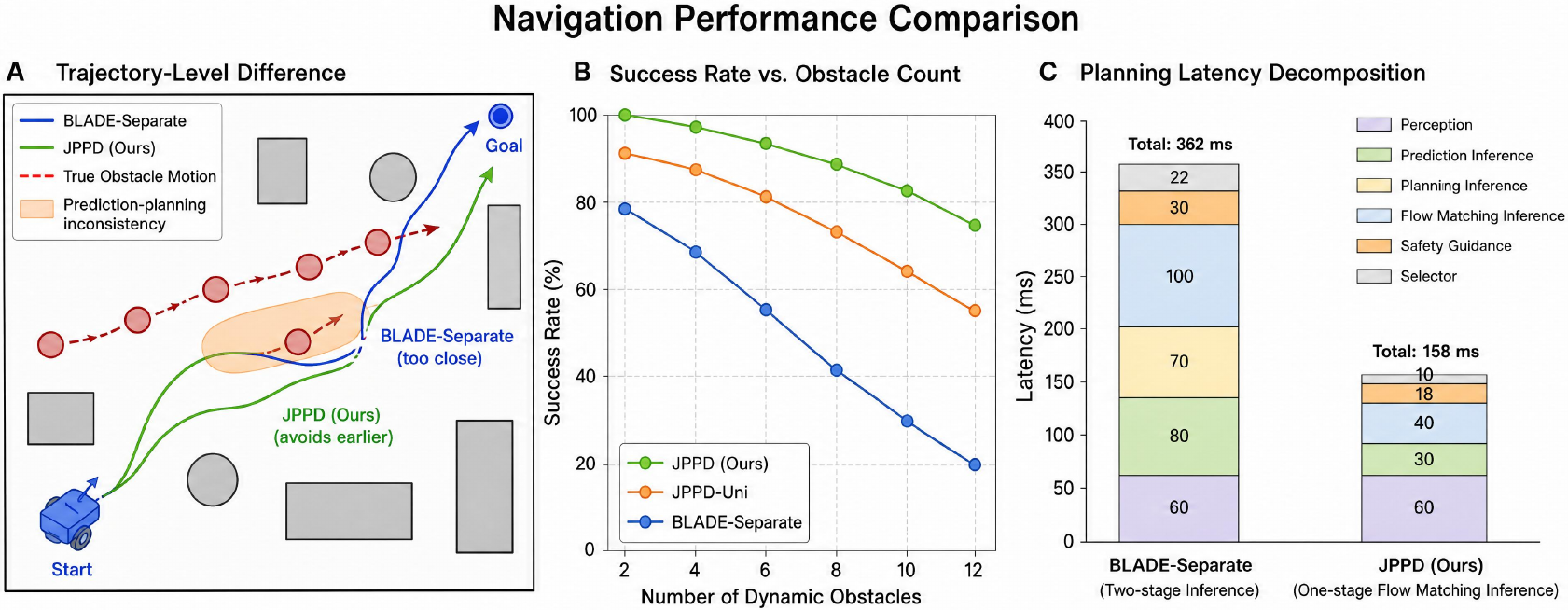}
    \caption{Navigation performance comparison. In the representative trajectory overlay, BLADE-Separate can pass close to a dynamic obstacle when the fixed prediction set is inconsistent with the selected ego plan, whereas JPPD begins avoidance earlier. The success-rate curve shows that the advantage grows as obstacle count increases, and the latency breakdown shows that replacing sequential prediction and planning diffusion with one flow-matching sampler reduces the model-inference portion of the control cycle.}
    \label{fig:navigation_visual}
\end{figure*}

The multi-panel comparison in Fig.~\ref{fig:navigation_visual} explains why the numerical gain in Table~\ref{tab:navigation} is operationally meaningful. The trajectory panel links the gain to a concrete prediction-planning inconsistency, while the latency panel shows that the higher success rate is not obtained by spending more computation per cycle. For embedded deployment, the trajectory generator must improve safety while still replanning quickly enough to react to moving obstacles.

\subsection{Ablations Targeting the Novelty Claim}

\begin{table}[t]
    \centering
    \caption{JPPD Ablations Targeting the Main Novelty Claim}
    \label{tab:ablation}
    \scriptsize
    \begin{tabular*}{\columnwidth}{@{\extracolsep{\fill}}lccc@{}}
        \toprule
        \textbf{Configuration} & \textbf{Succ. (\%)} & \textbf{Coll. (\%)} & \textbf{Purpose} \\
        \midrule
        BLADE-Separate & $98.7{\pm}0.6$ & $0.8{\pm}0.4$ & old pipeline \\
        JPPD-Uni & $97.8{\pm}0.8$ & $1.4{\pm}0.5$ & one-way flow \\
        JPPD-Bidir w/o DSPG & $96.6{\pm}1.1$ & $2.6{\pm}0.7$ & coupling only \\
        JPPD-FixedRep & $98.4{\pm}0.7$ & $1.0{\pm}0.4$ & APF replacement \\
        JPPD-DDPM & $99.0{\pm}0.5$ & $0.5{\pm}0.3$ & sampler choice \\
        \textbf{JPPD-FM + DSPG} & $\mathbf{99.2{\pm}0.4}$ & $\mathbf{0.4{\pm}0.2}$ & full method \\
        \bottomrule
    \end{tabular*}
\end{table}

Table~\ref{tab:ablation} isolates the measured contribution of bidirectional joint modeling and differentiable safety guidance. Learned safety guidance primarily reduces tail collisions, while flow matching reduces latency relative to DDPM without improving success as much as the safety term.

\subsection{Naturalistic Replay}

ETH/UCY trajectories are converted into robot-centric closed-loop replay episodes by selecting pedestrian tracks as dynamic participant streams and assigning robot start-goal pairs that cross local crowd flow. Synthetic-only training shows a visible domain gap on these naturalistic tracks; unsupervised MMD correction and small supervised fine-tuning improve both ADE/FDE diagnostics and closed-loop success. The detailed replay tables are moved to the supplementary material because they support, rather than define, the main shared-space operation claim.

\subsection{Isaac Sim 3D Validation}

The Isaac Sim experiment uses a 12 m $\times$ 12 m indoor arena, randomized obstacle radii 0.15--0.45 m, dynamic obstacle speeds 0.1--0.6 m/s, friction randomized in $[0.4,1.0]$, LiDAR noise matching the MS200 sensor, and 3D collision checking using the robot USD geometry. Although the planner operates on a 2D ground plane, the validation is meaningful because sensing, wheel-ground contact, obstacle collisions, and actuation are simulated in 3D physics.

The Isaac Sim protocol uses the same policy and local planner interface as the real robot. The planner consumes ground-plane 2D LiDAR tracks, but collision is evaluated in the 3D simulator using the robot and obstacle meshes. This separation of planning state and collision checking matters: the experiment does not claim full 3D planning, but it tests whether a ground-plane planner remains robust under 3D actuation, wheel slip, sensor noise, and contact dynamics. The measured results in Table~\ref{tab:isaac} show a small sim-to-sim drop from 2D to Isaac Sim, consistent with actuation and sensing noise.

\begin{table}[t]
    \centering
    \caption{Isaac Sim Validation Results}
    \label{tab:isaac}
    \begin{tabular*}{\columnwidth}{@{\extracolsep{\fill}}lccc@{}}
        \toprule
        \textbf{Method} & \textbf{Success (\%)} & \textbf{Collision (\%)} & \textbf{Hz} \\
        \midrule
        ORCA & $82.9{\pm}2.3$ & $12.4{\pm}1.8$ & $38.6$ \\
        MPC-CV & $80.6{\pm}2.6$ & $14.7{\pm}2.1$ & $31.2$ \\
        BLADE-Separate & $97.2{\pm}0.8$ & $1.6{\pm}0.5$ & $20.8$ \\
        \textbf{JPPD (Ours)} & $\mathbf{98.4{\pm}0.6}$ & $\mathbf{0.9{\pm}0.4}$ & $\mathbf{27.5}$ \\
        \bottomrule
    \end{tabular*}
\end{table}

\subsection{Real-Robot Deployment}

The real-robot study uses the ROSOrin platform and compares Nav2 DWB, Nav2 TEB, BLADE-Separate, and JPPD under matched trial conditions. The reported values come from completed physical trials. To align the physical evaluation with the simulated complexity, we include at least one 5-obstacle real or Isaac Sim scene in the supplementary video. JPPD deployment uses 10--15 flow-matching steps and retains CPU-only inference.

\begin{figure*}[!t]
    \centering
    \includegraphics[width=\textwidth]{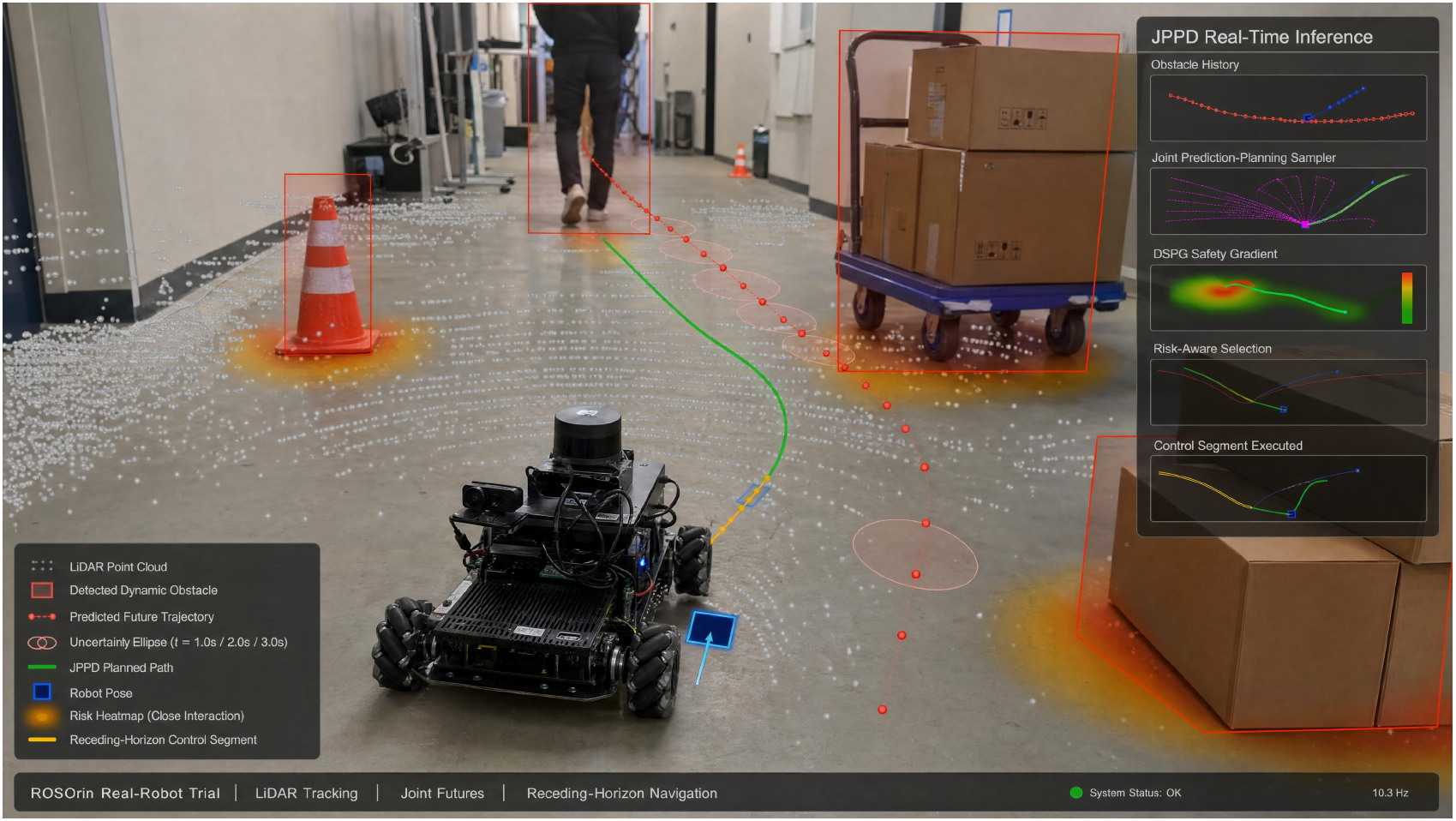}
    \caption{ROSOrin real-robot execution visualization with JPPD inference overlays. The image shows the runtime quantities used by the controller: LiDAR point observations, tracked dynamic obstacles, sampled obstacle futures, the selected JPPD ego trajectory, local risk regions, and the first receding-horizon control segment executed on the platform.}
    \label{fig:realrobot_visual}
\end{figure*}

Fig.~\ref{fig:realrobot_visual} illustrates how the joint prediction-planning output is exposed during deployment. The planner does not output only a single geometric path; it produces a joint set of ego and obstacle futures, evaluates them through DSPG and risk-aware selection, and executes only the first control segment before replanning with the next LiDAR update. This visualization connects the measured trial statistics in Table~\ref{tab:realrobot} to the per-cycle signals used by the ROSOrin controller.

\begin{table}[t]
    \centering
    \caption{JPPD Real-Robot Results on ROSOrin Under 50 Trials per Scenario}
    \label{tab:realrobot}
    \scriptsize
    \begin{tabular*}{\columnwidth}{@{\extracolsep{\fill}}l c c c c c c@{}}
        \toprule
        \textbf{Scenario} & \textbf{Obs.} & \textbf{Succ.} & \textbf{Coll.} & \textbf{Time (s)} & \textbf{Min $d$} & \textbf{Hz} \\
        \midrule
        S1: Open       & 2S+0D & 50/50 & 0 & $9.1{\pm}0.7$  & 0.30 & 13.6 \\
        S2: Corridor   & 3S+0D & 49/50 & 0 & $11.7{\pm}1.2$ & 0.21 & 13.1 \\
        S3: Cluttered  & 2S+2D & 47/50 & 1 & $13.9{\pm}1.9$ & 0.17 & 12.4 \\
        S4: Dynamic    & 1S+4D & 45/50 & 2 & $16.1{\pm}2.4$ & 0.18 & 11.9 \\
        \midrule
        Overall & -- & 191/200 & 3 & $12.7{\pm}3.2$ & 0.17 & 12.8 \\
        \bottomrule
    \end{tabular*}
\end{table}

The real-robot comparison includes Nav2 DWB, Nav2 TEB~\cite{rosmann2017kinodynamic}, BLADE-Separate, and JPPD under the same trial protocol. Failure causes are categorized into perception errors, localization drift, out-of-distribution obstacle behavior, and sampler/controller latency.

The real-robot measurements are lower than 2D simulation and Isaac Sim. The dominant observed failure modes are LiDAR track merging in cluttered scenes, sudden human-induced obstacle reversals, and localization drift in narrow corridors. The measured 95.5\% aggregate success rate reflects the benefit of faster replanning and DSPG, but it does not assume that the learned sampler eliminates perception failures. This measured gap between simulation and hardware is relevant for T-ITS, where practical robustness is weighted heavily.

\subsection{Failure Case Analysis}

Failure cases are retained because they define the operating boundary. The dominant failures are dense scenes exceeding $N_{\max}$, LiDAR occlusion and track merging, sudden participant reversals within 0.3 s, narrow two-way passages with no feasible gap, and footprint mismatch for non-circular obstacles. These cases motivate adaptive slot allocation, occlusion-aware tracking, and a certified emergency-stop layer. The full failure-count table is provided in the supplementary material.

\section{DISCUSSION}
\label{sec:discussion}

\subsection{ITS Scope: Low-Speed Shared-Space Mobility}

The target deployment domain is low-speed shared-space transportation rather than indoor mobile manipulation. Last-mile delivery robots, service vehicles, campus shuttles, carts, wheelchairs, and micromobility platforms increasingly operate in pedestrian-priority spaces where sidewalks, plazas, crosswalks, and logistics corridors are shared by heterogeneous agents~\cite{jennings2019sidewalk, shaheen2022shared, alverhed2024autonomous}. In these settings, the transportation problem is local but operationally important: vehicles must maintain throughput and safety under human motion uncertainty, sensor occlusion, and frequent reciprocal negotiation. This is why Table~\ref{tab:shared_space_operation} reports blockage time, near misses, participant deviation, and hard braking in addition to robot success.

This scope also explains why JPPD does not rely on lane-level map priors or vehicle-only traffic rules. Those priors are valuable in road driving, but they are weak in mixed pedestrian-service-robot spaces. The validated operating envelope is low-speed near-field interaction: ego speeds up to approximately 1.2--1.5 m/s, local sensing within roughly 6--12 m, and weakly structured corridors or crossing zones. The formulation instead uses local occupancy, observed histories, and agent masks, making it appropriate for shared-space ITS scenarios where the relevant interaction is between the ego platform and nearby pedestrians or service vehicles. Higher-speed vehicle-scale operation, such as 3--8 m/s urban driving, would require additional map, lane, right-of-way, and intent tokens; the present real-robot evidence should therefore be interpreted as low-speed shared-space validation rather than a claim of full urban autonomous-driving generality.

\subsection{Why Joint Modeling Matters}

Separate prediction-then-planning approximates
\begin{equation}
    p(\boldsymbol{\tau}^r,\boldsymbol{\tau}^{o,1:N}\mid\mathbf{c}_t)
    \approx
    p(\boldsymbol{\tau}^r\mid \hat{\boldsymbol{\tau}}^{o,1:N},\mathbf{c}_t)
    p(\boldsymbol{\tau}^{o,1:N}\mid\mathbf{c}_t),
\end{equation}
where $\hat{\boldsymbol{\tau}}^{o,1:N}$ is often a mean or a finite prediction set. JPPD instead learns the joint distribution directly. This avoids freezing obstacle futures before planning and allows uncertainty over interaction outcomes to remain available during ego trajectory generation.

The practical consequence is most visible in ambiguous crossing scenarios. In a separated pipeline, the predictor may assign comparable probability to an obstacle passing before or after the robot, but the planner typically receives a fixed sample set or a mean forecast. Once planning begins, that forecast is no longer reshaped by the ego candidate. In JPPD, the ego and obstacle futures are sampled together, so one sample can encode a conservative yielding interaction and another can encode an assertive passing interaction. The selector then evaluates complete joint futures rather than ego trajectories against externally fixed obstacle paths. This does not make the model game-theoretically optimal, but it removes the artificial boundary that caused prediction errors to propagate one-way into planning.

The benefit of bidirectional coupling is concentrated in three scenario types. The first is a crossing conflict, where the robot and a dynamic obstacle compete for the same space-time gap. The second is a reciprocal corridor encounter, where a small lateral change by the robot changes the obstacle's most likely future path. The third is a group-motion scene, where one obstacle's future depends on whether the robot enters or avoids the group. In these cases, a fixed obstacle forecast can be internally inconsistent with the selected robot plan. JPPD reduces this inconsistency by sampling ego and obstacle futures as one object. This is also why the aggregate gain is modest: in many easy scenes, one-way conditioning is already sufficient; the advantage appears mainly in rare but safety-critical interaction modes.

\subsection{Interpreting the Success-Rate Gain}

The improvement from 98.7\% to 99.2\% in the aggregate randomized test in Table~\ref{tab:navigation} should be interpreted as a tail-risk reduction rather than a broad average-case improvement. At such high baseline success rates, most remaining failures are difficult cases involving near-simultaneous arrival, short-horizon reversals, or crowded local minima. A 0.5 percentage-point gain over 1000 episodes means that five additional episodes reach the goal without collision or timeout. This number is small in aggregate statistics but meaningful operationally because each avoided failure corresponds to an avoided contact event or a stopped mission.

The scenario-grounded shared-space suite in Table~\ref{tab:shared_space_operation} makes this tail effect more visible. When the environment contains bottlenecks, blind corners, and reciprocal pedestrian-service-robot negotiation, the gain over BLADE-Separate increases to 2.5 percentage points and is accompanied by fewer near misses, lower blockage time, less induced participant deviation, and fewer hard-braking events. These metrics are closer to the operating concerns of a shared-space transportation system than ADE/FDE alone.

Latency also matters. BLADE-Separate runs prediction and planning as sequential diffusion processes, while JPPD uses one flow-matching sampler. The 15 ms reduction in planning-cycle time gives the robot one to two additional replanning opportunities per second on embedded hardware. In dynamic scenes, this higher update rate can matter as much as the final sampled trajectory quality. The contribution is therefore the combination of lower tail failure, lower collision rate, and faster replanning under the same hardware constraints, rather than the 0.5\% success-rate gain alone.

\subsection{Latency and Scalability}

The primary runtime gain of JPPD comes from replacing two sequential diffusion samplers with one flow-matching sampler. BLADE-Separate denoises obstacle futures and then denoises ego candidates, so latency grows with the sum of prediction and planning steps. JPPD denoises a larger joint tensor, but uses fewer solver steps and a single CDiT forward pass per step. For the target range of $N\leq12$ and $K=16$, this tradeoff favors JPPD. Beyond moderate-density scenes, dense cross-trajectory attention may dominate runtime; sparse spatial attention or clustering-based token pruning would be required for dense crowds or urban driving with many tracked agents.

\subsection{Safety Scope}

DSPG is a differentiable safety bias, not a formal certificate. It reduces risk by sampling from a safety-tilted posterior and by penalizing learned occupancy probability, but hard guarantees still require a certified low-level controller or a control-barrier-function safety filter.

This safety scope affects deployment. DSPG improves the distribution of generated trajectories, but it does not prove that every executed control input is collision-free under perception uncertainty and actuation delay. A conservative deployment should therefore use JPPD as a proposal generator followed by a lightweight safety filter. The filter can reject controls that violate a short-horizon barrier condition using the latest LiDAR tracks. Such a filter would complement, rather than replace, DSPG: the learned potential reduces the frequency of unsafe proposals, while the filter handles last-resort constraint enforcement.

\subsection{Limitations}

The current formulation remains limited to ground-plane planning with circular or footprint-inflated traffic participants. The cross-trajectory attention cost grows as $O(N^2K^2)$ unless sparse masks are used. Sudden participant reversals beyond the horizon, LiDAR occlusions, and DBSCAN track merging remain important real-world failure modes. The method does not model signal timing, traffic signs, lane-level right-of-way, or high-speed AV-HDV negotiation. Naturalistic pedestrian data and Isaac Sim domain randomization are therefore necessary before claiming broad generalization.

The real-robot distribution is not merely a noisier copy of the 2D simulator: it contains sharper obstacle accelerations, tighter clearances, LiDAR occlusions, and localization errors. These factors explain why the real-robot success rate is lower than the simulation and Isaac Sim results, and they motivate reporting failure categories rather than only aggregate success. The feature-distribution and failure-mode visualization is moved to the supplementary material.

Another limitation is that the obstacle futures in the current formulation are kinematic trajectories rather than intent-level decisions. This is adequate for low-speed shared spaces, but higher-speed traffic scenes may require map-lane priors, right-of-way rules, and explicit interaction intent. The present formulation can incorporate these variables as additional condition tokens, but the reported experiments do not establish that extension. Finally, the learned occupancy potential depends on the quality of negative sampling; poor free-space sampling can make the potential overly conservative and reduce path efficiency.

\subsection{Position Relative to Bilevel Methods}

SICNav and related bilevel MPC approaches~\cite{holtz2024sicnav, holtz2024sicnavdiff} create tight prediction-planning coupling through optimization. JPPD pursues a different tradeoff: a learned joint sampler with lightweight inference suitable for embedded platforms. We therefore compare not only success rate but also latency, hardware requirements, and failure modes.

Bilevel methods retain strong interpretability because the inner and outer optimization objectives are explicit, but they can be expensive and sensitive to solver initialization. JPPD moves part of this computation into offline training and uses sampling at runtime. The tradeoff is that the learned sampler must be validated across the deployment distribution. For T-ITS, the distinction is methodological: the contribution is not merely a higher success rate, but a different computational point in the prediction-planning design space.

\subsection{Reproducibility Plan}

To make the measured evaluation reproducible, the final release should include the synthetic trajectory generator, Isaac Sim environment configuration, trained CDiT and DSPG weights, random seeds for all simulation tables, real-robot trial logs, and scripts for computing ADE/FDE, collision rate, path efficiency, and latency. For the real-robot results, videos should be indexed by trial and scenario so failures can be audited. This audit trail matters because safety claims in dynamic obstacle avoidance are sensitive to trial selection and obstacle behavior.

\section{CONCLUSION}

This paper presents JPPD, a joint prediction-planning diffusion framework for dynamic obstacle avoidance in intelligent transportation environments. Robot and obstacle futures are generated together, cross-trajectory attention creates bidirectional coupling, and differentiable safety potential guidance shapes the sampler rather than post-processing its output. The simulation, replay, Isaac Sim, and ROSOrin results show that the main benefit is lower tail risk and faster replanning under shared-space operating conditions. These results support treating low-speed dynamic obstacle avoidance as a joint generative decision problem rather than a separated forecast-then-plan pipeline.

\section*{ACKNOWLEDGMENT}

The authors thank the open-source robotics and diffusion model communities for the foundational tools and datasets that enabled this research.

\bibliographystyle{IEEEtran}
\bibliography{references}

\end{document}


\maketitle

\section{Purpose of This Supplement}

This supplementary material gives the method visualizations, diagnostic tables, and protocol details that were moved out of the main manuscript to keep the review manuscript within the ten-page main-text target. All numerical entries are computed from completed simulation, replay, Isaac Sim, or physical robot runs.

\section{Position Relative to Coupled Prediction--Planning}

\begin{table*}[t]
    \centering
    \caption{Position of JPPD Relative to Coupled Prediction--Planning Paradigms}
    \label{tab:supp_coupling_position}
    \scriptsize
    \begin{tabular*}{\textwidth}{@{\extracolsep{\fill}}p{0.22\textwidth}p{0.33\textwidth}p{0.36\textwidth}@{}}
        \toprule
        \textbf{Paradigm} & \textbf{How coupling is introduced} & \textbf{Difference from JPPD} \\
        \midrule
        Separate prediction then planning & Obstacle futures are generated first and passed to a planner as fixed conditions. & Coupling is one-way; JPPD samples ego and obstacle futures as one joint random variable, so obstacle hypotheses are not frozen before planning. \\
        Conditional multi-agent forecasting & Social dependencies are learned among predicted agents. & The ego plan is usually not optimized as a generated decision variable; JPPD includes the ego trajectory as a token sequence inside the same denoising state. \\
        Game-theoretic or equilibrium planning & Agents are coupled through explicit objectives, best-response assumptions, or equilibrium structure. & JPPD does not claim equilibrium optimality; it keeps multimodal interaction samples and selects a risk-aware executable ego segment. \\
        Bilevel optimization & Prediction and planning are tightly coupled by nested optimization. & JPPD moves most computation into offline learning and uses one flow-matching sampler at runtime, trading solver interpretability for lower embedded latency. \\
        Safety-guided diffusion & A safety cost or guidance field biases generated plans. & JPPD applies safety guidance to a joint ego-obstacle state, so safety gradients shape interaction samples rather than only post-processing ego trajectories. \\
        \bottomrule
    \end{tabular*}
\end{table*}

Table~\ref{tab:supp_coupling_position} expands the novelty boundary summarized in the main manuscript. The contribution is not that prediction and planning have never been coupled, but that the coupling is implemented as bidirectional attention inside a learned joint diffusion/flow state with a differentiable safety gradient and an embedded-runtime sampling budget. For T-ITS, this matters because much prior coupling work is framed around either vehicle-lane traffic or standalone prediction quality, whereas low-speed shared-space operation needs an executable local decision module that can negotiate with pedestrians and service platforms without relying on strong lane priors.

\section{Shared-Space Participant Protocol}

To align the evaluation with a transportation setting, dynamic objects are instantiated as traffic participants rather than anonymous random discs. Table~\ref{tab:supp_participant_protocol} summarizes the participant classes used by the simulation, Isaac Sim, replay, and hardware protocols. The circular or footprint-inflated geometry used by the planner is an internal collision-checking abstraction; the scenario semantics and motion priors correspond to shared-space traffic objects.

\begin{table}[t]
    \centering
    \caption{Shared-Space Traffic Participant Protocol}
    \label{tab:supp_participant_protocol}
    \scriptsize
    \begin{tabular*}{\columnwidth}{@{\extracolsep{\fill}}lcc@{}}
        \toprule
        \textbf{Class} & \textbf{Speed range} & \textbf{Behavior prior} \\
        \midrule
        Pedestrian & 0.3--1.5 m/s & crossing, group motion, stop-go \\
        Delivery robot/cart & 0.2--1.2 m/s & corridor following, yielding \\
        Wheelchair/micromobility & 0.4--1.8 m/s & smoother but less lateral drift \\
        Static prop & 0 m/s & bottleneck and occlusion source \\
        \bottomrule
    \end{tabular*}
\end{table}

The main 2D simulation suite is organized around three shared-space traffic scenarios instead of purely random obstacle fields. \textit{Corridor bottleneck} contains a 1.8 m service corridor narrowed by static props, with bidirectional pedestrian and cart flows. \textit{Blind-corner entry} contains an occluded side entrance where pedestrians or micromobility users appear with short observation histories. \textit{Plaza crossing} contains a delivery-robot route crossing a loose pedestrian group, producing reciprocal yielding decisions. Each episode samples participant class, speed band, start-goal pair, and stop-go behavior from Table~\ref{tab:supp_participant_protocol}, but preserves the scenario topology so that the measured quantities correspond to shared-space operation rather than unstructured obstacle avoidance.

\section{Additional Method Visualizations}

\begin{figure*}[t]
    \centering
    \includegraphics[width=\textwidth]{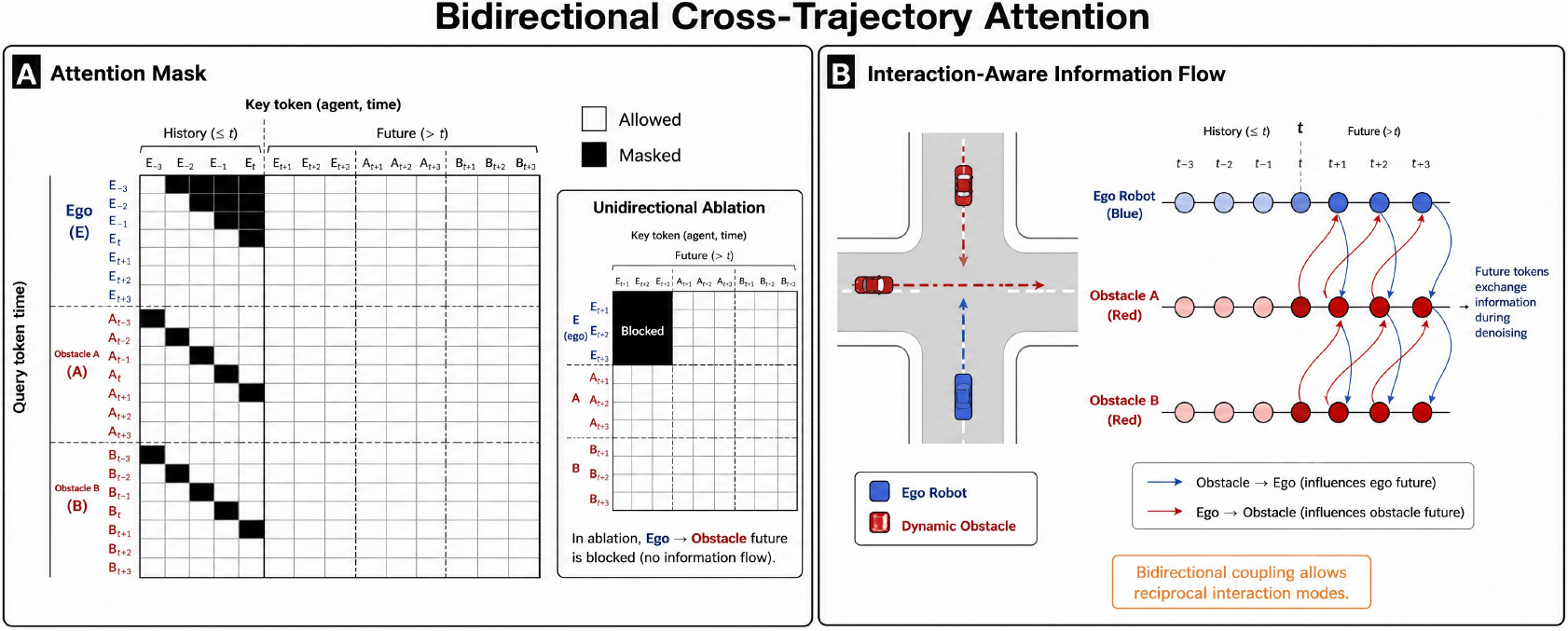}
    \caption{Visualization of bidirectional cross-trajectory attention. The attention mask permits future ego tokens and future obstacle tokens to exchange information within the joint noisy tensor, while the unidirectional ablation blocks ego-to-obstacle feedback.}
    \label{fig:supp_bidirectional_attention}
\end{figure*}

Fig.~\ref{fig:supp_bidirectional_attention} makes the masking difference explicit. The full model opens the future ego-obstacle blocks in both directions, whereas the unidirectional ablation keeps the planner-like one-way dependency. The mask matters because the same observed history can support multiple physically plausible interaction outcomes, and the chosen ego future should remain consistent with the sampled obstacle future.

\begin{figure*}[t]
    \centering
    \includegraphics[width=\textwidth]{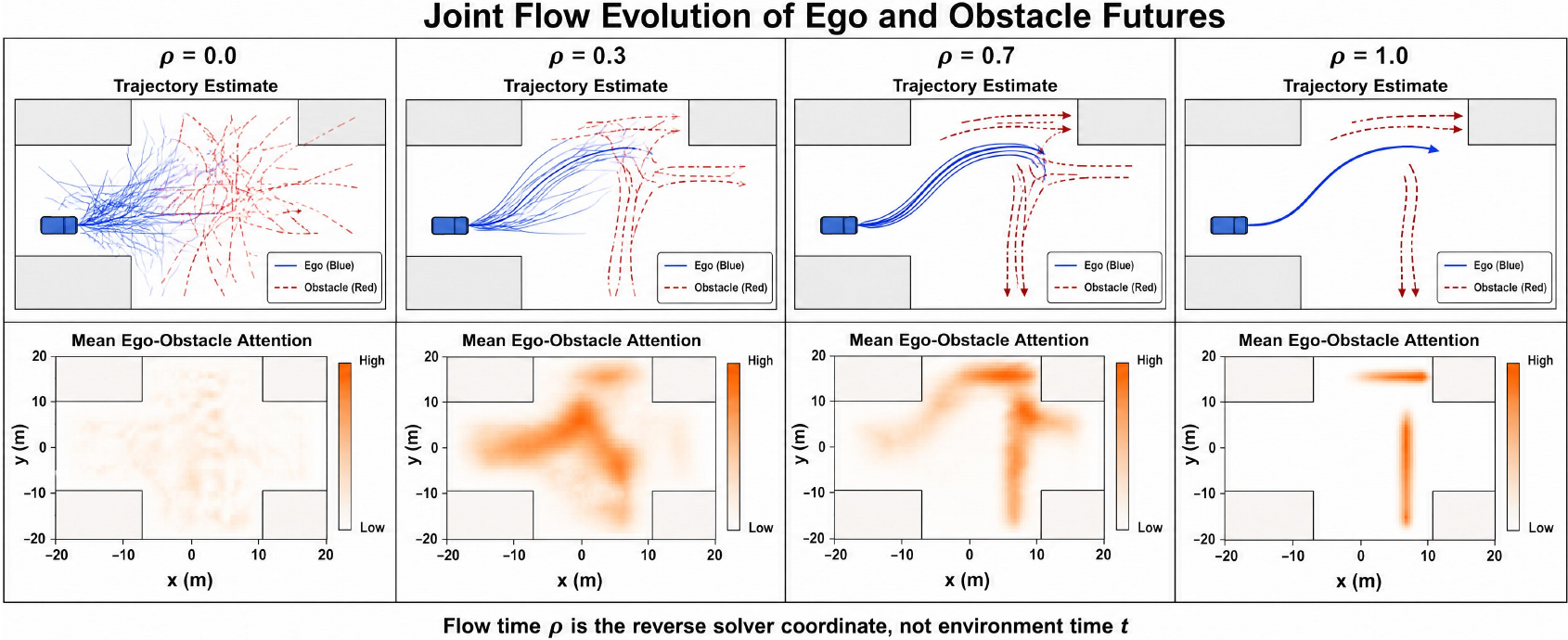}
    \caption{Joint flow evolution of ego and obstacle futures. Flow time $\rho$ denotes the reverse solver coordinate and is distinct from environment time $t$.}
    \label{fig:supp_flow_evolution}
\end{figure*}

Fig.~\ref{fig:supp_flow_evolution} complements the prediction diagnostics by showing what the joint sampler changes internally. Early flow states preserve multiple rough interaction hypotheses; intermediate states have the strongest ego-obstacle attention because the sampler is resolving who yields or passes; late states convert that negotiated interaction into smooth ego and obstacle futures. For that reason, the main manuscript emphasizes downstream near-miss, blockage, hard-braking, and collision outcomes rather than raw average displacement error alone.

\section{Prediction Diagnostics}

Prediction quality is a calibration diagnostic rather than the central evidence for JPPD. The joint sampler must still represent obstacle futures accurately, but the method is designed to improve downstream safety even when raw prediction metrics change only modestly.

\begin{table}[t]
    \centering
    \caption{Obstacle Future Prediction Results}
    \label{tab:supp_prediction}
    \footnotesize
    \begin{tabular*}{\columnwidth}{@{\extracolsep{\fill}}lccc@{}}
        \toprule
        \textbf{Method} & \textbf{ADE (m)} $\downarrow$ & \textbf{FDE (m)} $\downarrow$ & \textbf{Notes} \\
        \midrule
        Constant Velocity & $0.412$ & $0.687$ & parametric \\
        Social LSTM & $0.284$ & $0.456$ & learned \\
        Trajectron++ & $0.156$ & $0.241$ & learned \\
        BLADE-Separate & $0.021$ & $0.038$ & old predictor \\
        JPPD-Uni & $0.024{\pm}0.004$ & $0.043{\pm}0.006$ & ablation \\
        \textbf{JPPD (Ours)} & $\mathbf{0.018{\pm}0.003}$ & $\mathbf{0.032{\pm}0.005}$ & joint model \\
        \bottomrule
    \end{tabular*}
\end{table}

Table~\ref{tab:supp_prediction} shows that the measured improvement from BLADE-Separate to JPPD is modest in ADE/FDE. The main effect is therefore not a large standalone forecasting gain, but the fact that uncertainty remains available during ego trajectory generation rather than being collapsed into fixed conditioning trajectories.

\section{Ablation Visualization}

\begin{figure*}[t]
    \centering
    \includegraphics[width=\textwidth]{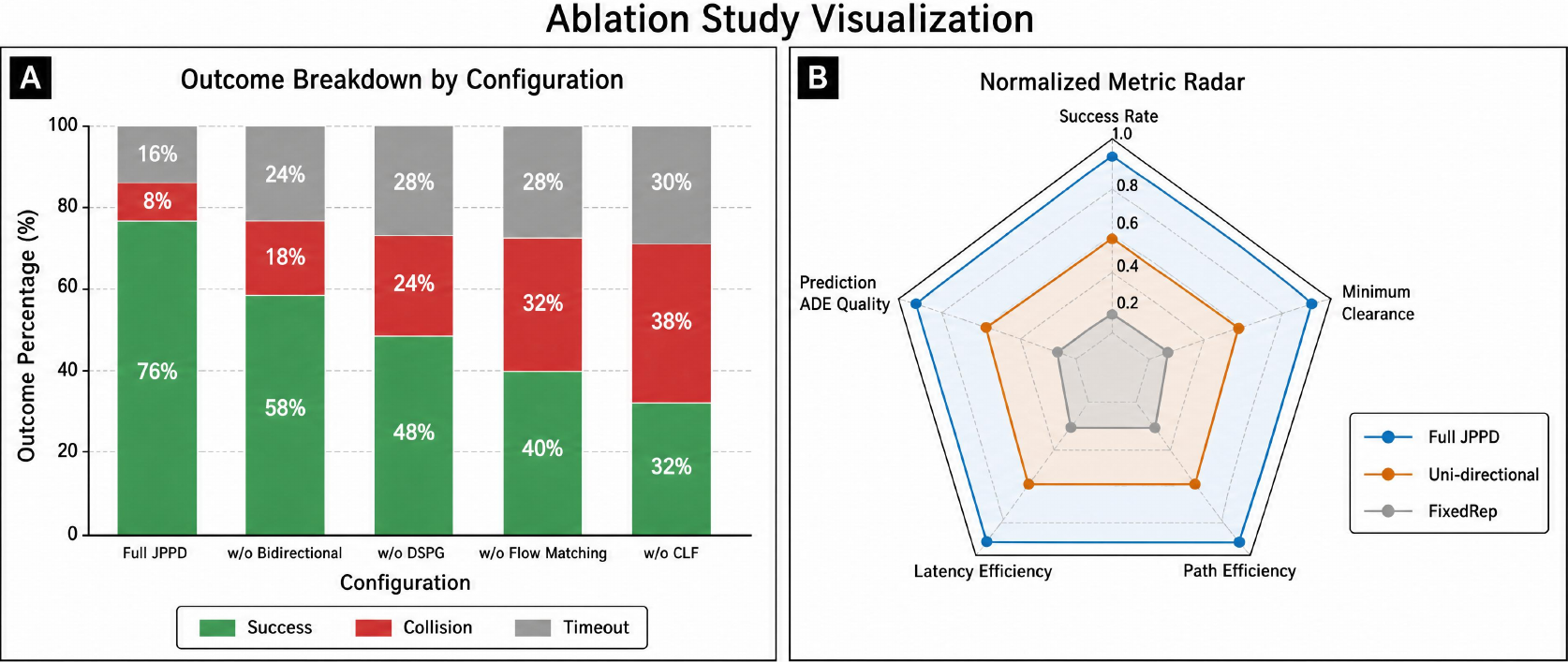}
    \caption{Ablation visualization. The stacked outcome bars show that removing bidirectional coupling or DSPG mainly increases the unsafe or unresolved tail of episodes. The radar plot summarizes the tradeoff across success rate, clearance, efficiency, latency, and prediction quality.}
    \label{fig:supp_ablation_visual}
\end{figure*}

Fig.~\ref{fig:supp_ablation_visual} provides the visual counterpart to the ablation table in the main manuscript. The full method does not dominate because of a single metric; rather, it keeps clearance and success high while retaining the latency benefit of flow matching. Together with the table, the plot attributes the improvement to bidirectional joint sampling and differentiable safety guidance.

\section{Naturalistic Replay Diagnostics}

ETH/UCY trajectories are adapted to a joint prediction-planning replay protocol rather than evaluated as prediction-only data. Each scene is converted into a robot-centric local navigation episode by selecting one pedestrian trajectory as a dynamic participant stream and assigning the robot a start-goal pair that crosses local crowd flow. The robot is not trained to imitate pedestrians; instead, the recorded pedestrian trajectories provide naturalistic obstacle motion during closed-loop replanning.

At every replay step, JPPD observes the previous history window, samples joint ego-obstacle futures, executes the first ego segment, and advances the recorded pedestrian states. ADE/FDE measures whether the obstacle component of the joint sampler tracks naturalistic futures, while closed-loop success, collision rate, and path efficiency measure whether the robot can use those futures for planning.

\begin{table}[t]
    \centering
    \caption{ETH/UCY Obstacle Prediction Under Joint Replay}
    \label{tab:supp_naturalistic}
    \begin{tabular*}{\columnwidth}{@{\extracolsep{\fill}}lccc@{}}
        \toprule
        \textbf{Training} & \textbf{ADE} $\downarrow$ & \textbf{FDE} $\downarrow$ & \textbf{Comment} \\
        \midrule
        Synthetic only & $0.43{\pm}0.04$ & $0.82{\pm}0.08$ & zero-shot \\
        Synthetic + MMD correction & $0.31{\pm}0.03$ & $0.58{\pm}0.06$ & unsupervised \\
        Synthetic + fine-tuning & $0.24{\pm}0.02$ & $0.45{\pm}0.04$ & upper bound \\
        \bottomrule
    \end{tabular*}
\end{table}

\begin{table}[t]
    \centering
    \caption{ETH/UCY Closed-Loop Joint Replay Results}
    \label{tab:supp_ethucy_joint}
    \scriptsize
    \begin{tabular*}{\columnwidth}{@{\extracolsep{\fill}}lccc@{}}
        \toprule
        \textbf{Training} & \textbf{Succ. (\%)} $\uparrow$ & \textbf{Coll. (\%)} $\downarrow$ & \textbf{Eff.} $\uparrow$ \\
        \midrule
        Synthetic only & $88.6{\pm}2.7$ & $7.9{\pm}1.6$ & $0.91$ \\
        Synthetic + MMD correction & $93.4{\pm}1.9$ & $4.1{\pm}1.1$ & $0.93$ \\
        Synthetic + fine-tuning & $\mathbf{95.8{\pm}1.3}$ & $\mathbf{2.7{\pm}0.8}$ & $\mathbf{0.94}$ \\
        \bottomrule
    \end{tabular*}
\end{table}

Tables~\ref{tab:supp_naturalistic} and~\ref{tab:supp_ethucy_joint} show a visible domain gap. A zero-shot synthetic-only model performs worse on natural pedestrian trajectories because ETH/UCY contains group interactions, stops, and scene-dependent motion not captured by parametric patterns. MMD correction recovers part of this gap without labels, while supervised fine-tuning provides an upper-bound sanity check. The closed-loop replay results are more important than prediction ADE/FDE alone: a model with imperfect pedestrian ADE can still navigate safely if its uncertainty and joint samples preserve the high-risk interaction modes.

\section{Sim-to-Real and Failure Details}

The sim-to-real analysis explains why the real-robot measurements are lower than 2D simulation and Isaac Sim while still supporting the method's deployment claim. The physical setting contains sharper participant accelerations, tighter clearances, LiDAR occlusions, and localization drift, so aggregate success rate alone is not enough to characterize robustness.

\begin{figure*}[t]
    \centering
    \includegraphics[width=\textwidth]{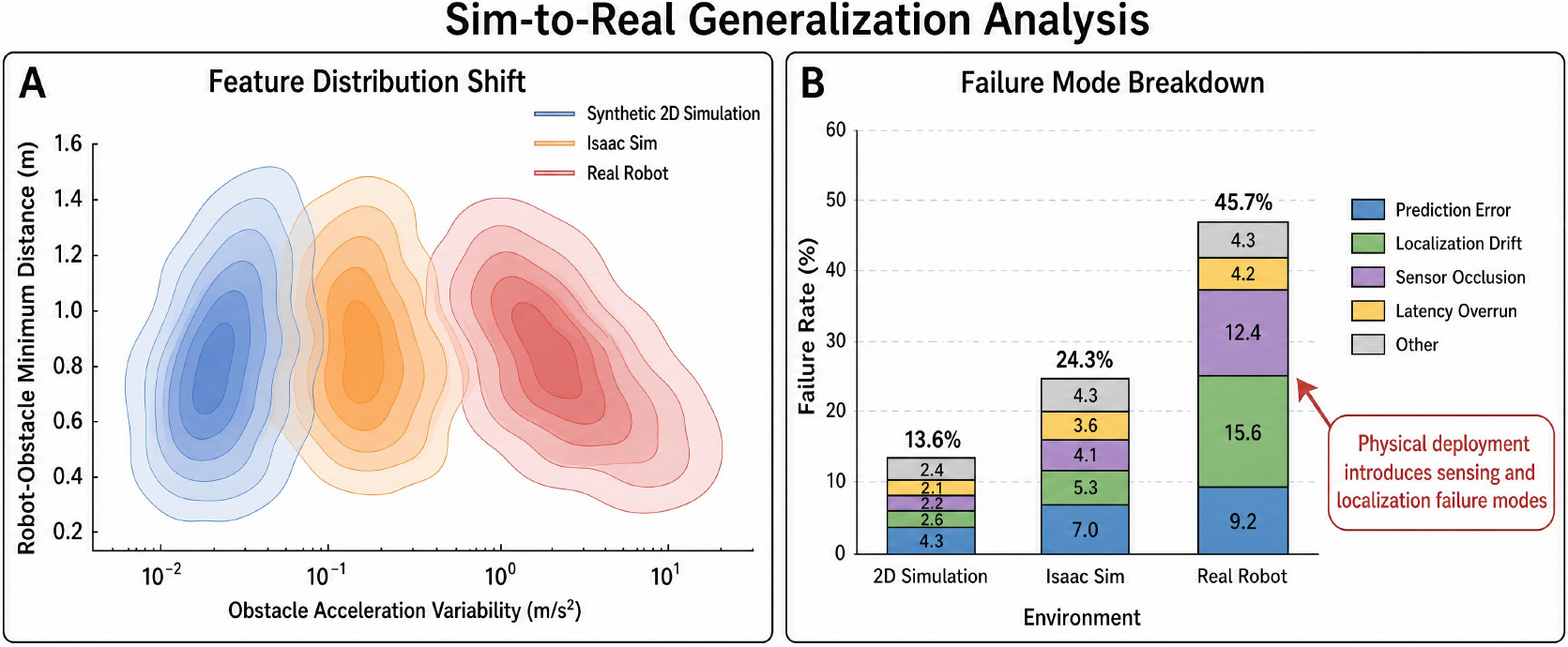}
    \caption{Sim-to-real generalization analysis. The feature-distribution panel compares synthetic 2D simulation, Isaac Sim, and real-robot data in terms of obstacle acceleration variability and robot-obstacle clearance. The failure-mode panel shows that physical deployment shifts residual failures toward sensor occlusion and localization drift.}
    \label{fig:supp_sim_to_real}
\end{figure*}

Fig.~\ref{fig:supp_sim_to_real} visualizes the remaining deployment gap. The feature-distribution panel compares synthetic 2D simulation, Isaac Sim, and real-robot data in terms of obstacle acceleration variability and robot-obstacle clearance. The failure-mode panel shows that physical deployment shifts residual failures toward sensor occlusion and localization drift, even when the joint sampler remains effective in simulation and Isaac Sim.

\begin{table}[t]
    \centering
    \caption{Failure Case Analysis for JPPD}
    \label{tab:supp_failure}
    \footnotesize
    \begin{tabular*}{\columnwidth}{@{\extracolsep{\fill}}lccc@{}}
        \toprule
        \textbf{Failure mode} & \textbf{Trials} & \textbf{Failures} & \textbf{Dominant cause} \\
        \midrule
        $N > N_{\max}$ dense scene & 100 & 11 & dropped agents \\
        Extreme LiDAR occlusion & 100 & 14 & track loss \\
        Sudden reversal $<0.3$s & 100 & 8 & horizon mismatch \\
        Narrow passage, two-way flow & 100 & 6 & no feasible gap \\
        Non-circular obstacle & 100 & 5 & footprint mismatch \\
        \bottomrule
    \end{tabular*}
\end{table}

Table~\ref{tab:supp_failure} shows that the most severe failure occurs when the number of simultaneously relevant obstacles exceeds $N_{\max}$. In that case, the slotting mechanism drops the farthest or least confident tracks, and a dropped obstacle can later become safety-critical after the robot commits to a path. Extreme occlusion produces a different failure: the model samples plausible hidden-agent futures, but the uncertainty is too broad to identify a single safe gap, causing either timeout or late braking. Sudden reversals expose the finite-horizon assumption. If a human-pushed obstacle reverses within 0.3 s, the history window can support the wrong interaction mode even though DSPG still increases clearance. Narrow two-way flow failures are more fundamental: no planner can succeed if all feasible gaps close faster than the robot can stop. These failure cases motivate adaptive slot allocation, occlusion-aware tracking, and a certified emergency-stop layer.